\documentclass[letterpaper]{article} 
\usepackage{aaai24}  
\usepackage{times}  
\usepackage{helvet}  
\usepackage{courier}  
\usepackage[hyphens]{url}  
\usepackage{graphicx} 
\usepackage{natbib}  
\usepackage{caption} 
\usepackage{algorithm}
\usepackage{algorithmic}

\usepackage{newfloat}
\usepackage{listings}
\DeclareCaptionStyle{ruled}{labelfont=normalfont,labelsep=colon,strut=off} 
\floatstyle{ruled}
\newfloat{listing}{tb}{lst}{}
\floatname{listing}{Listing}
\usepackage{lineno}
\usepackage{amsmath}
\usepackage{amsfonts}
\usepackage{algorithm}
\usepackage{algorithmic}
\usepackage{array}
\usepackage{graphicx}
\usepackage{subfigure}
\usepackage{booktabs}
\usepackage{multirow}
\usepackage{subfigure}
\usepackage{caption}
\usepackage{graphicx} 
\usepackage{makecell}
\usepackage{amsmath}
\usepackage{amsthm}
\usepackage{amssymb}

\newcommand{\argmin}{\operatornamewithlimits{arg\,min}}
\newcommand{\xx}{\boldsymbol x}
\newcommand{\yy}{\boldsymbol y}

\newcommand{\rhox}{{\rho_{X}}}
\newcommand{\Ltwo}{{L^2(\mathbb{P})}}

\newtheorem{theorem}{Theorem}

\newtheorem{definition}{Definition}
\newtheorem{assumption}{Assumption}

\newcommand{\Nystrom}[1]{{Nystr\"om}}

\newcommand{\HH}{\mathcal H}

\newcommand{\la}{\lambda}

\newcommand{\X}{{X}}
\newcommand{\XX}{{\boldsymbol X}}

\newcommand{\RR}{\mathbb{R}}

\newcommand{\bpr}{\begin{proof}}
		\newcommand{\epr}{\end{proof}}
\newcommand{\be}{\begin{equation}}
		\newcommand{\ee}{\end{equation}}

\newcommand{\mh}{{\mathcal{H}_K}}
\newcommand{\mO}{\mathcal{O}}

\newcommand{\frho}{f^*}
\newcommand{\wrho}{\ww^*}

\newcommand{\fkrr}{f_{\md,\lambda}}
\newcommand{\wkrr}{\ww_{\md,\lambda}}

\newcommand{\md}{\mathcal{D}}
 
\newcommand{\mX}{\mathcal{X}}
\newcommand{\mY}{\mathcal{Y}}

\newcommand{\Hessian}{\HH_{\md, t}}

\newcommand{\Hessiannoniidjt}{\HH_{\mdnoniidj, t}}

\newcommand{\HessiannoniidjSketch}{{\boldsymbol \Upsilon}_{\mdnoniidj, \la}}
\newcommand{\HessianSketch}{\widetilde{\HH}_{\md, t}}

\newcommand{\gradient}{\gg_{\md, t}}

\newcommand{\gradientnoniidj}{\gg_{\mdnoniidj, t}}
\newcommand{\iteration}{\sum_{j=1}^m \frac{n_j}{N}}

\newcommand{\mdnoniidj}{\mathcal{D}_j}

\newcommand{\PhiD}{\phi({\boldsymbol X})}

\newcommand{\PhinoniidDj}{\phi({\boldsymbol X}_j)}

\def\HH{{\boldsymbol H}}
\def\SS{{\boldsymbol S}}

\def\ww{{\boldsymbol w}}
\def\xx{{\boldsymbol x}}
\def\yy{{\boldsymbol y}}

\def\gg{{\boldsymbol g}}
\def\oo{{\boldsymbol \omega}}

\title{FedNS: A Fast Sketching Newton-Type Algorithm for Federated Learning}
\author {
    Jian Li\textsuperscript{\rm 1},
    Yong Liu\textsuperscript{\rm 2}\thanks{Corresponding author},
    Wei Wang\textsuperscript{\rm 3},
    Haoran Wu\textsuperscript{\rm 3},
    Weiping Wang\textsuperscript{\rm 1}
}
\affiliations {
    \textsuperscript{\rm 1}Institute of Information Engineering, Chinese Academy of Sciences\\
    \textsuperscript{\rm 2}Gaoling School of Artificial Intelligence, Renmin University of China\\
    \textsuperscript{\rm 3}China Unicom Research Institute\\
    lijian9026@iie.ac.cn, liuyonggsai@ruc.edu.cn, wangweiping@iie.ac.cn
}

\begin{document}

\nocopyright

\maketitle

\begin{abstract}
    Recent Newton-type federated learning algorithms have demonstrated linear convergence with respect to the communication rounds. However, communicating Hessian matrices is often unfeasible due to their quadratic communication complexity. In this paper, we introduce a novel approach to tackle this issue while still achieving fast convergence rates. Our proposed method, named as Federated Newton Sketch methods (FedNS), approximates the centralized Newton's method by communicating the sketched square-root Hessian instead of the exact Hessian. To enhance communication efficiency, we reduce the sketch size to match the effective dimension of the Hessian matrix. We provide convergence analysis based on statistical learning for the federated Newton sketch approaches. Specifically, our approaches reach super-linear convergence rates w.r.t. the communication rounds for the first time. We validate the effectiveness of our algorithms through various experiments, which coincide with our theoretical findings.
\end{abstract}


\section{Introduction}
Due to the huge potential in terms of privacy protection and reducing computational costs, Federated Learning (FL) \cite{konevcny2016federated,mcmahan2017communication,li2020federated,wei2021federated,wei2022non,li2024optimaldistributed} becomes a promising framework for handling large-scale tasks.
In federated learning, a key problem is to achieve a tradeoff between the convergence rate and the communication burdens.

First-order optimization algorithms have achieved great success in federated learning, including FedAvg (LocalSGD) \cite{mcmahan2017communication} and FedProx \cite{li2020federated}.
These methods communicate the first-order information rather than the data across local machines, which protect the privacy training data and allow the data heterogeneity to some extent.
Despite recent efforts and progress on the convergence analysis  \cite{li2018federated,li2019convergence,karimireddy2020scaffold,pathak2020fedsplit,glasgow2021sharp} and the generalization analysis \cite{mohri2019agnostic,li2023optimal,li2023optimalnystr,su2021achieving,yuan2021we} of FedAvg and FedProx, the convergence rate of first-order federated algorithms is still slow, i.e., a sublinear converge rate $\mO(1/t)$, where $t$ is the communication rounds.

In the traditional centralized learning, with some mild conditions, second-order optimal algorithms \cite{boyd2004convex,bottou2018optimization}, for example (quasi) Newton's methods, can achieve at least a linear convergence rate.
The compute of inverse Hessian is time consuming, and thus many classic approximate Newton's methods are proposed, including BFGS \cite{broyden1970convergence}, L-BFGS \cite{liu1989limited}, inexact Newton \cite{dembo1982inexact}, Gauss-Newton \cite{schraudolph2002fast} and Newton sketch \cite{pilanci2017newton}.
However, if we directly apply Newton's method to federated learning, the communication complexity of sharing local Hessian matrices is overwhelming.

Indeed, first-order algorithms characterize low communication burdens but slow convergence rates, while Newton's methods lead to fast convergence rates but with high communication complexity.
To take advantage of both first-order and second-order algorithms, we propose a federated Newton sketch method, named \texttt{FedNS}, which shares both first-order and second-order information across devices, i.e., local gradients and sketched square-root Hessian.
Using line search strategy and adaptive learning rates, we devise a dimension-efficient approach \texttt{FedNDES}.
We then study the convergence properties for the proposed algorithms. 
We conclude with experiments on publicly available datasets that complement our theoretical results, exhibiting both computational and statistical benefits. 
We leave proofs in the appendix\footnote{Full version: \url{https://lijian.ac.cn/files/2024/FedNS.pdf}\\
Code: \url{https://github.com/superlj666/FedNS}}.
We summarize our contributions as below:

\textbf{1) On the algorithmic front.} We propose two fast second-order federated algorithms, which improve the approximation of the centralized Newton's method while the communication costs are favorable due to the small sketch size.
Specifically, using line search and adaptive step-sizes in Algorithm \ref{alg.FedNDES}, the sketch size can be reduced to the effective dimension of the Hessian matrix.

\textbf{2) On the statistical front.}
Drawing upon the established outcomes from the centralized sketching Newton literature \cite{pilanci2017newton,lacotte2021adaptive}, we furnish convergence analysis for two federated Newton sketching algorithms: \texttt{FedNS} and its line-search variant \texttt{FedNDES}.
These methodologies delineate not only super-linear convergence rates but also entail a small sketch size, corresponding to the effective dimension of the Hessian in the case of \texttt{FedNDES}.
Note that the proposed algorithms achieve super-linear convergence rates while upholding a reasonable level of communication complexity.
\subsection{Related Work}

\textbf{FedAvg and FedProx.} 
FedAvg and FedProx only share local gradients of the size $M$ and the communication complexity is $\mO(M)$, where $M$ is the dimension of feature space.
The convergence properties of FedAvg and FedProx have been well-studied in \cite{li2019convergence,li2020federated,su2021achieving}, of which the iteration complexity is $T = \mO(1/\delta)$ to obtain some $\delta$-accurate solution.

\textbf{Newton sketch.}
Newton sketch was proposed in \cite{pilanci2017newton} to accelerate the compute of Newton's methods.
It has been further extended, for example, Newton-LESS \cite{derezinski2021newton} employed leverage scores to sparsify the Gaussian sketching matrix and \cite{lacotte2020effective,lacotte2021adaptive} proved the sketch size can be as small as the effective dimension of the Hessian matrix.
With the sketching Newton's methods, the communication complexity is $\mO(kM)$, where $k$ is the sketch size.
In some cases, it leads to a super-linear convergence $T = \mO (\log (\log 1/\delta))$ for a $\delta$-accurate solution.

\textbf{Newton-type federated algorithms.}
DistributedNewton \cite{ghosh2020distributed} and LocalNewton \cite{gupta2021localnewton} perform Newton's method instead of SGD on local machines to accelerate the convergence of local models.
FedNew \cite{elgabli2022fednew} utilized one pass ADMM on local machines to calculating local directions and approximate Newton method to update the global model.
FedNL \cite{safaryan2021fednl} sended the compressed local Hessian updates to global server and performed Newton step globally.
Based on eigendecomposition of the local Hessian matrices, SHED \cite{fabbro2022newton} incrementally updated eigenvector-eigenvalue pairs to the global server and recovered the Hessian to use Newton method.

\section{Backgrounds}
A standard federated learning system consists of a global server and $m$ local compute nodes.
On the $j$-th worker $\forall ~ j \in [m]$, the local training data $\mdnoniidj=\{(\xx_{ij}, y_{ij})\}_{i=1}^{{n_j}}$ is drawn from a local distribution $\rho_j$ on $\mX \times \mY$, where $\mX$ is the input space and $\mY$ is the output space.
We denote the disjoint union of local training data $\md = \bigcup_{j=1}^m \mdnoniidj$ as the entire training data that corresponds to a global distribution $\rho$ on $\mX \times \mY$.
For the sake of privacy preservation and efficient distributed computation, a federated learning system aims to train a global model without sharing local data.
The ideal empirical model is trained on the entire training dataset $\md$ w.r.t. the training objective.
We denote $n_j := |\mdnoniidj|$ the number of local examples on the $j$-th machine and $N := |\md| = \sum_{j=1}^m n_j$ the number of all train examples.

\subsection{Centralized Newton's Method}
The objective of centralized learning on the entire train set $\md$ can be stated as $f(\xx) = \langle \ww, \phi(\xx)\rangle$ with
\begin{align}
    \label{objective.general}
    \wkrr = \argmin_{\ww \in \mh} ~ \underbrace{\frac{1}{N}\sum_{i=1}^{N} \ell\left(f(\xx_i), y_i\right) + {\lambda} \alpha(\ww)}_{L(\md, \ww)},
\end{align}
where $(\xx_i, y_i) \in \md$, $\ell$ is the loss function, $\alpha(\ww)$ is the penalty term and $\lambda > 0$ is the regularity parameter.
We assume that $f$ belongs to the reproducing kernel Hilbert space (RKHS) $\mh$ defined by a Mercer kernel $K: \mX \times \mX \to \RR$.
Throughout, we denote the inner product $K(\xx, \xx') = \langle \phi(\xx), \phi(\xx')\rangle_K$ and the corresponding norm by $\|\cdot\|_K$, where $\phi: \mX \to \mh$ is the implicit feature mapping.
The reproducing property guarantees kernel methods $f: \mX \to \mY$ admitting
\begin{align}
    \label{eq.kernel_trick}
    f(\xx) = \langle \ww, \phi(\xx)\rangle_K, \qquad \forall ~ \ww \in \mh, \xx \in \mX.
\end{align}

If $L$ is twice differentiable convex function in terms of $\ww$, the centralized problem \eqref{objective.general} on $\md$ can be solved by the exact Newton's method
\begin{align}
    \label{eq.motivation.centralized}
    \ww_{t+1} = \ww_t - \mu \Hessian^{-1} \, \gradient,
\end{align}
where $\mu$ is the step-size and the gradient and the Hessian matrix are computed by
\begin{align*}
    &\gradient := \nabla L(\md, \ww_t) + \lambda \nabla \alpha(\ww_t), \\
    &\Hessian := \nabla^2 L(\md, \ww_t) + \lambda \nabla^2 \alpha(\ww_t).
\end{align*}
Let the feature mapping be finite dimensional $\phi: \mX \to \RR^M$. Then, there are finite dimensional gradients $\gradient \in \RR^M$ and Hessian matrices $\Hessian \in \RR^{M \times M}$.

Since the federated learning system cannot visit local training data, it fails to compute the global gradient $\gradient$ and Hessian matrix $\Hessian$ directly.
The total time complexity for Newton's method is $\mO(N M^2t)$ with a super-linear convergence rate $t=\log(\log(1/\delta))$ for $\delta$-approximation guarantee \cite{dennis1996numerical}, i.e., $L(\md, \ww_t) - L(\md, \wkrr) \leq \delta$.

\subsection{Newton's Method with Partial Sketching}
To improve the computational efficiency of Newton's method, \cite{pilanci2017newton} proposed Newton sketch to construct a structured random embedding of the Hessian matrix with a sketch matrix $\SS \in \RR^{k \times M}$ where $k \ll N$.
Instead of sketching the entire Hessian matrix, partial Newton sketch \cite{pilanci2017newton,lacotte2021adaptive} only sketched the loss term $\nabla^2 L(\md, \ww_t) \approx (\SS \nabla^2 L(\md, \ww_t)^{1/2})^\top (\SS \nabla^2 L(\md, \ww_t)^{1/2})$ while reserving the exact Hessian for the regularity term 
$\nabla^2\alpha(\ww_t)$.
The partial Newton sketch can be stated as
\begin{equation}
    \begin{aligned}
        \label{eq.newton_sketch}
        \ww_{t+1} = &~ \ww_t - \mu \HessianSketch^{-1} \, \gradient, \qquad \text{with} \\
        \HessianSketch = &~(\SS \nabla^2 L(\md, \ww_t)^{1/2})^\top (\SS \nabla^2 L(\md, \ww_t)^{1/2}) \\
        & + \lambda \nabla^2 \alpha(\ww_t).
    \end{aligned}
\end{equation}
Note that, the sketch matrix $\SS \in \RR^{k \times M}$ are zero-mean and normalized $\mathbb{E} [\SS^\top \SS / k] = \mathbf{I}_M$. 
Different types of randomized sketches lead to different the sketch times \cite{lacotte2021adaptive}, i.e., $\mO(NM k)$ for sub-Gaussian embeddings, $\mO(NM \log k)$ for subspace randomized Hadamard transform (SRHT), and $\mO(\textbf{nnz}(\HessianSketch))$ for the sparse Johnson-Lindenstrauss transform (SJLT).
For the SJLT, only one non-zero entry exists per column, which causes much faster sketch time but requires a larger subspace.
Throughout this work, we focus on the SRHT because it achieves a tradeoff between fast sketch time and small sketch dimension \cite{ailon2006approximate}.

There are various examples for the optimization problem \eqref{objective.general} where the regularization term is $\lambda$-strongly convex and partial sketched square-root Hessian $S \nabla^2 L(\ww)$ is amenable to fast computation.
For example, the widely used ridge regression and regularized logistic regression. 
More examples are referred to Section 3.1 \cite{lacotte2021adaptive}.

\section{Federated Learning with Newton's Methods}
To reduce the communication burdens in FedNewton, we propose a communication-efficient algorithm to approximate the global Hessian matrix.
Based on the Newton Sketch \cite{pilanci2017newton}, we devise an efficient Newton sketch method for federated learning, which performs an approximate local Hessians using a randomly projected or sub-sampled Hessian on the local workers.
Then, it summarizes the local Hessian matrices to approximate the global Hessian matrix and then performs Newton's method on the approximate global Hessian matrix.

\subsection{Federated Newton's Method (FedNewton)}
To approximate the global model $\fkrr$ well, the federated learning algorithms usually share local information to the other clients, i.e. first-order gradient information in FedAvg and FedProx.
We consider using both first-order and second-order information to characterize the exact solution of \eqref{objective.general} on the entire dataset $\md$.

Noting that, both local gradients and Hessians in \eqref{eq.motivation.centralized} can be summarized up to the global gradient and Hessian, respectively. 
This motivates us to summarize local gradients and Hessians to conduct federated Newton's method (FedNewton) 
\begin{equation}
    \begin{aligned}
        \label{eq.FedNewton}
        &\ww_{t+1} = \ww_t - \mu \Hessian^{-1} \, \gradient \qquad \text{with} \\
        &\Hessian = \iteration \Hessiannoniidjt, \quad \gradient = \iteration \gradientnoniidj.
    \end{aligned}
\end{equation}

\textbf{Complexity Analysis.}
Before the iterations, the computation of feature mapping on any local machine consumes $\mO({n_j} M d)$ time. 
On the $j$-th local worker, the time complexity for iteration is at least $\mO({n_j}M^2 + {n_j}M)$ to compute $\Hessiannoniidjt$ and $\gradientnoniidj$, respectively. 
And the time complexity is $\mO(mM^2 + M^3)$ to summarize local Hessians and compute the global inverse Hessian.
However, the communication cost is $\mO(M^2)$ to upload local Hessian matrices in each iteration, which is infeasible in the practical federated learning scenarios.
The total computational complexity is $\mO(\max_{j\in[m]} {n_j} M d + {n_j}M^2 t + M^3t)$ and the communication complexity is $\mO(M^2t)$, where Newton's method achieves a quadratic convergence $t = \mO(\log(\log(1/\delta)))$ for $\delta$-approximation guarantee \cite{dennis1996numerical}, i.e., $L(\md, \ww_t) - L(\md, \wkrr) \leq \delta$, where $\wkrr$ is the empirical risk minimizer \eqref{objective.general}.

\begin{algorithm}[t]
    \caption{\small Federated Learning with Newton Sketch (\texttt{FedNS})}
    \label{alg.FedNS}
    \begin{algorithmic}[1]
        \REQUIRE  
        Feature mapping $\phi: \RR^d \to \RR^M$, start point $\ww_0$, the termination iterations $T$ and the step-size $\mu.$
        \ENSURE The global estimator $\ww_T$.
        \STATE \textbf{Local machines:} Compute the local feature mapping data matrix $\PhinoniidDj$.
        \FOR{$t = 1$ to $T$}            
            \STATE \textbf{Local machines:} Sample the sketch matrix $\SS_j^t \in \RR^{k \times n_j}$ from the SRHT.
            Compute local sketch Hessian matrices $\HessiannoniidjSketch = \SS_j^t \nabla^2 L(\mdnoniidj, \ww_t)^{1/2}$ and local gradients $\gradientnoniidj$. 
            Upload them to the global server $(\uparrow)$.
            \STATE \textbf{Global server:} Compute the global Hessian matrix $\HessianSketch = \iteration \HessiannoniidjSketch^\top \HessiannoniidjSketch + \lambda \nabla^2 \alpha(\ww_t)$ and the global gradient $\gradient = \iteration \gradientnoniidj$ and update the global estimator $$\ww_{t} = \ww_{t-1} - \mu\HessianSketch^{-1} ~ \gradient$$ and communicate it to local machines $(\downarrow)$.
        \ENDFOR
    \end{algorithmic}
\end{algorithm}  

\subsection{Federated Learning with Newton Sketch (\texttt{FedNS})}
Now suppose the local Hessian matrix square-root $\nabla^2 L(\md, \ww_t)^{1/2}$ of dimensions ${n_j} \times M$ is available, from \eqref{eq.newton_sketch}, we obtain local sketch Hessian on the empirical loss term
\begin{align*}
    \HessiannoniidjSketch = \SS_j \nabla^2 L(\mdnoniidj, \ww_t)^{1/2}.
\end{align*}
Here, $\SS_j \in \RR^{k \times {n_j}}$ is the sketch matrix for the $j$-th worker with $k \ll {n_j}$ and thus the communicated sketch Hessian is with a small size $\HessiannoniidjSketch \in \RR^{k \times M}$.
From \eqref{eq.FedNewton},
we approximate the global Hessian by summarizing local sketching Hessian matrices
\begin{align}
    \label{eq.hessian_sketch}
    \HessianSketch = \iteration \HessiannoniidjSketch^\top \HessiannoniidjSketch + \lambda \nabla^2 \alpha(\ww_t).
\end{align}
Here, we approximate the global Hessian with local sketch Hessian matrices $\Hessian \approx \HessianSketch.$
We approximate the exact Newton's method on the entire data by local Newton sketch
\begin{align}
    \label{eq.FedNS}
    \ww_{t+1} = \ww_t - \mu \HessianSketch^{-1} ~ \gradient
\end{align}
where $\HessianSketch$ is the approximate global Hessian from \eqref{eq.hessian_sketch} and $\gradient$ is the global gradient from \eqref{eq.FedNewton}.

We formally introduce the general framework for Federated Newton Sketch (\texttt{FedNS}) in Algorithm \ref{alg.FedNS}.
Note that, using sketch Hessian, we communicate $\HessiannoniidjSketch$ of the dimensions $k \times M$ instead of the exact local Hessian $\Hessiannoniidjt$ of the dimensions $M^2$.
Since local gradients $\gradientnoniidj$ are $M$ dimensional vectors, the uploads of sketch Hessian matrices dominate the communication complexity.

\textbf{Complexity Analysis.}
Before the iterations, the computation of feature mapping on any local machine consumes $\mO({n_j} M d)$ time.
On the $j$-th local worker, the sketching time is $\mO(n_j M \log k)$ for SRHT, while it is $\mO(n_j M k)$ for classic sub-Gaussian embeddings.
On the global server, the time complexity of \texttt{FedNS} is $\mO(mkM^2 + M^3)$ to obtain the global sketch Hessian and its inverse.
Nevertheless, \texttt{FedNS} reduces the communication burdens from $\mO(M^2)$ to $\mO(k M)$.
Overall, \texttt{FedNS} not only speedup the local computations, but more importantly reduce the communication costs.
The total computational complexity is $\mO(\max_{j\in[m]} {n_j} M d + n_j M t \log k + mkM^2t + M^3t)$.
More importantly, the communication complexity is reduced from $\mO(M^2 t)$ in FedNewton to $\mO(k M t)$ in \texttt{FedNS}.
The sketch Newton method leads to a super-linear convergence rate $t = \mO(\log(1/\delta))$ for $\delta$-approximation guarantee \cite{pilanci2017newton,lacotte2021adaptive}.
Note that, from Theorem \ref{thm.FedNS}, since $k=\Omega(M)$ for \texttt{FedNS}, just using sketching without line-search does not reduce communication cost.
We then present line-search step based sketching Federated Newton method, which can guarantee smaller communication costs.

\subsection{Federated Newton's Method with Dimension-Efficient Sketching (\texttt{FedNDES})}
From the theoretical results in the next section, the sketch size for \texttt{FedNS} is $k \simeq M$ to achieve a super-linear convergence rate, which is still too expensive when $M$ is large to approximate the kernel.
Since the communication is determined by the sketch size, we devise a dimension-efficient federated Newton sketch approach (\texttt{FedNDES}), shown in Algorithm \ref{alg.FedNDES} to reduce the sketch size, based on the Newton sketch with backtracking line (Armijo) search \cite{boyd2004convex,nocedal2006numerical,pilanci2017newton,lacotte2021adaptive}.
Backtracking line search begins with an initial step-size $\mu$ and backtracks until the adjusted linear estimate overestimates the loss function.  For more information refer to \cite{boyd2004convex}.

Different from \texttt{FedNS}, Algorithm \ref{alg.FedNDES} applies the two phases updates with different sketch sizes $\overline{m}_1$ and $\overline{m}_2$.
The following theoretical results illustrate both these two sketch sizes depend on the effective dimension that is much smaller than $M$ and $N$.
The approximate Newton decrement $\tilde{\lambda}(\ww_t)$ is used as an exit condition and the threshold for the adaptive sketch sizes to guarantee smaller iterations and sketch sizes.

\textbf{Complexity Analysis.}
Since the compute of approximate global Hessian and the inverse of it dominate the training time, the total computational complexity for \texttt{FedNDES} is similar to \texttt{FedNS} that is $\mO(\max_{j\in[m]} {n_j} M d + n_j M t \log k + mkM^2t + M^3t)$.
However, the communication complexity $\mO(k M t)$ is reduced owing to a smaller sketch size in \texttt{FedNDES}.
The sketch size is $k \simeq M$ for \texttt{FedNS}, while it is $k \simeq \text{Tr}(\HessianSketch(\HessianSketch + \lambda I)^{-1})$ for \texttt{FedNDES} that is much smaller than $M$.
Smaller sketch size makes the Newton methods practical for multiple communications in federated learning settings.
Meanwhile, even with the smaller sketch size, \texttt{FedNDES} can still achieve the super-linear convergence rate $t = \mO(\log \log(1/\delta))$ for $\delta$-approximation solution.

\begin{algorithm}[t]
    \caption{\small Dimension-efficient federated Newton (\texttt{FedNDES})}
    \label{alg.FedNDES}
    \begin{algorithmic}[1]
        \REQUIRE Feature mapping $\phi: \RR^d \to \RR^M$, start point $\ww_0$, accuracy tolerance $\delta > 0$, line-search parameters $(a, b)$, threshold sketch sizes $\overline{m}_1$ and $\overline{m}_2$, and the decrement parameter $\eta$.
        \ENSURE The global estimator $\ww_T$.
        \STATE \textbf{Local machines:} Compute the local feature mapping data matrix $\PhinoniidDj$. 
        Initialize and $k_t = \overline{m}_1$. 
        \FOR{$t = 1, \cdots, T$}            
            \STATE \textbf{Local machines:} Sample the sketch matrix $\SS_j^t \in \RR^{k \times n_j}$ from the SRHT.
            Compute local sketch square root Hessian $\HessiannoniidjSketch = \SS_j^t \nabla^2 L(\mdnoniidj, \ww_t)^{1/2}$ and local gradients $\gradientnoniidj$. 
            Upload them to the global server $(\uparrow)$.
            \STATE \textbf{Global server:} Compute the global Hessian matrix $\HessianSketch = \iteration \HessiannoniidjSketch^\top \HessiannoniidjSketch + \lambda \nabla^2 \alpha(\ww_t)$, the global gradient $\gradient = \iteration \gradientnoniidj$ and the approximate Newton step $\Delta\ww_t = - \HessianSketch^{-1} ~ \gradient$.
            Compute the approximate Newton decrement 
            $$
                \tilde{\lambda}(\ww_t) = \gradient^\top \Delta\ww_t.
            $$
            {If} $\tilde{\lambda}(\ww_t)^2 \leq \frac{3}{4} \delta$ return the model $\ww_t$. {Otherwise} send $\Delta\ww_t$ and $\tilde{\lambda}(\ww_t)$ to local workers.
            \STATE \textbf{Local machines:} Line search from $\mu_j = 1$: 
            \textbf{while} $L(\mdnoniidj, \ww_t + \mu_j \Delta\ww_t) > L(\mdnoniidj, \ww_t) + a \mu_j \tilde{\lambda}(\ww_t)$, \textbf{then} $\mu_j \leftarrow b\mu_j$.
            Send $\mu_j$ to the global server.
            \STATE \textbf{Global server:} 
            Let $\mu = \min_{j \in m} \mu_j.$
            Update the global estimator $$\ww_{t} = \ww_{t-1}  + \mu \Delta\ww_t.$$
            {If} $\tilde{\lambda}(\ww_t) > \eta$, set $k = \overline{m}_1$. Otherwise, set $k = \overline{m}_2.$
            Communicate the global model $\ww_{t}$ and the sketch size $k$ to local machines $(\downarrow)$.
        \ENDFOR
    \end{algorithmic}
\end{algorithm}

\section{Theoretical Guarantees}

  Before the generalization analysis for algorithms, we start with some notations and assumptions.

  \begin{assumption}[Twice differentiable and convex]
      \label{asm.differentiable}
      The loss function and regularity function $\ell, \alpha: \RR^d \to \RR$ are both closed and twice differentiable convex functions and $\nabla^2 \alpha(\ww) \succeq I_d$.
  \end{assumption}
  
  \begin{assumption}[Strongly convextiy and smoothness]
      \label{asm.convex}
      Let $\gamma = \lambda_\text{min}(\nabla^2 L(\ww))$ be the minimum eigenvalue and $\beta = \lambda_\text{max}(\nabla^2 L(\ww))$ be the maximum eigenvalue of the Hessian.
  \end{assumption}
  
  In the standard analysis of Newton's method, $\gamma$ and $\beta$ are commonly used to measure the strong convexity and smoothness of the objective function $L$ \cite{pilanci2017newton}.

  \begin{assumption}[Lipschitz continuous Hessian]
      \label{asm.lipschitz}
      The Hessian map is Lipschitz continuous with modulus $G$, i.e.,
      $
          \|\nabla^2 L(\ww) - \nabla^2 L(\ww')\| \leq G \|\ww - \ww'\|_2.
      $
  \end{assumption}
  
  The above assumptions are standard conditions for both convex and non-convex optimization.
  Under these conditions, with the appropriate initialization $\|\ww - \wkrr\| \leq \frac{\gamma}{2G}$, the Newton approach can guarantee a quadratic convergence \cite{boyd2004convex}.  
  Let $\ww_t$ be the federated Newton model defined in \texttt{FedNS} or \texttt{FedNDES}, and $\wrho$ be the target model. 
  It holds the following error decomposition 
  \begin{equation}
      \begin{aligned}
          \label{eq.error_decomposition}
          \|\ww_t - \wrho\|_2 
          \leq \underbrace{\|\ww_t - \wkrr\|_2}_\text{federated error} + \underbrace{\|\wkrr - \wrho\|_2}_\text{centralized excess risk},
      \end{aligned}
  \end{equation}
  where $\wkrr$ is the centralized ERM model on the training data $\md$.
  Since the centralized excess risk $L(\wkrr) - L(\wrho)$ is standard in statistical learning \cite{bartlett2002rademacher,caponnetto2007optimal}, we focused on the federated error term.

  \subsection{Convergence Analysis for \texttt{FedNS}}
  Theorem 1 in \cite{pilanci2017newton} provided the convergence analysis for \textit{centralized} Newton sketch in \eqref{eq.newton_sketch}.
  Based on it, we present on the generalization analysis for \textit{federated} Newton sketch in Algorithm \ref{alg.FedNS}.
  \begin{theorem}[Convergence guarantees of FedNS]
      \label{thm.FedNS}
      Let $\delta \in (0, 1)$.
      Under Assumptions \ref{asm.differentiable}, \ref{asm.convex}, \ref{asm.lipschitz}, FedNS updates in Algorithm \ref{alg.FedNS} based on an appropriate initialization $\|\ww_0 - \wkrr\|_2 \leq \frac{\delta \gamma}{8 G}$.
      Using the iteration-dependent sketching accuracy $\epsilon = \frac{1}{\log(1+t)}$ and sketch size $k = \Omega(M)$, with the probability at least $1 - c_1 e^{- c_2 k \epsilon^2}$, we have
      \begin{align*}
          \|\ww_t - \wkrr\|_2 \leq 
          \frac{1}{\log(1+t)} \frac{\beta}{\gamma} \|\ww_{t-1} - \wkrr\|_2 .
      \end{align*}
      This guarantee a super-linear convergence rate, since $\lim_{t \to \infty} \frac{\|\ww_t - \wkrr\|_2}{\|\ww_{t-1} - \wkrr\|_2} = 0.$
      Here, $\wkrr$ is the centralized model on $\md$ and $\ww_t$ is the federated model trained in Algorithm \ref{alg.FedNS}.
  \end{theorem}

The above theorem shows that, depending on the structure of the problem $\gamma, \beta$, \texttt{FedSN} achieves the super-linear convergence rate using a sufficient sketch size $k \gtrsim M$.
And thus, the communication complexity in each iteration is at least $\mO(M^2)$, the same as FedNewton.
There are two drawbacks in Theorem \ref{thm.FedNS}: 1)
The analysis depends on some constants from the properties of $L$, i.e., the curvature constants $\gamma, \beta$ and the Lipschitz constant $G$, which are usually unknown in practice.
2) Theorem \ref{thm.FedNS} requires a initialization condition $\|\ww_0 - \wkrr\|_2 \leq \frac{\delta \gamma}{8 G}$. However, it is hard to find a appropriate start point $\ww_0$ satisfying the condition in practice.
3) The communication complexity of \texttt{FedSN} $\mO(kMt)$ depends on the sketch size, but the communication of current sketch size $k \gtrsim M$ is overly expensive.

\begin{table*}[t]
    \renewcommand{\arraystretch}{1.2}
    \resizebox{\textwidth}{!}{
    \centering
    \begin{tabular}{m{5cm}|m{2cm}|m{2cm}|m{2cm}|m{2.2cm}|m{4cm}}
        \toprule
        Related Work & Heterogeneous setting & Sketch size $k$ & Iterations $t$ & Communication per round & Total communication complexity  \\ \toprule
        FedAvg \cite{li2019convergence,su2021achieving} & $\surd$ & $-$ & $\mO\left(\frac{1}{\delta}\right)$ & $\mO(M)$ & $\mO(\frac{M}{\delta})$ \\ \hline
        FedProx \cite{li2020federated,su2021achieving} & $\surd$ & $-$ & $\mO\left(\frac{1}{\delta}\right)$ & $\mO(M)$ & $\mO(\frac{M}{\delta})$ \\ \hline
        DistributedNewton \cite{ghosh2020distributed} & $\times$ & $-$ & $\mO\left(\log \frac{1}{\delta}\right)$ & $\mO(M)$ & $\mO(M \log \frac{1}{\delta})$ \\ \hline
        LocalNewton \cite{gupta2021localnewton} & $\times$ & $-$ & $\mO\left(\log \frac{1}{\delta}\right)$ & $\mO(M)$ & $\mO(M \log \frac{1}{\delta})$ \\ \hline
        FedNL \cite{safaryan2021fednl} & $\surd$ & $-$ & $\mO\left(\log \frac{1}{\delta}\right)$ & $\mO(M)$ & $\mO(M \log \frac{1}{\delta})$ \\ \hline
        SHED \cite{fabbro2022newton} & $\surd$ & $-$ & $\mathcal{O}(\log \frac{1}{\delta})$ & $-$ & $\mO(M^2)$ \\ \hline
        FedNewton & $\surd$ & $-$ & $\mO\left(\log\log \frac{1}{\delta}\right)$ & $\mO(M^2)$ & $\mO(M^2 \log \log \frac{1}{\delta})$ \\ \hline
        \texttt{FedNS} (Algorithm \ref{alg.FedNS}) & $\surd$ & $M$ & $\mO\left(\log \log \frac{1}{\delta}\right)$ & $\mO(kM)$ & $\mO(kM \log \log \frac{1}{\delta})$ \\ \hline
        \texttt{FedNDES} (Algorithm \ref{alg.FedNDES}) & $\surd$ & $\tilde{d}_\lambda$ & $\mO\left(\log \log \frac{1}{\delta}\right)$ & $\mO(k M)$ & $\mO(k M \log \log \frac{1}{\delta})$ \\
        \bottomrule
    \end{tabular}
    }
    
    \small
    Note: The computational complexities are computed in terms of regularized least squared loss to obtain a $\delta$-accurate solution, i.e., $L(\ww_t) - L(\wkrr) \leq \delta$.
    The convergence analysis for FedAvg is provided in \cite{li2019convergence,su2021achieving} and that for FedProx is provided in \cite{li2020federated,su2021achieving}.
    
    \caption{Summary of communication properties for related work.}
    \label{tab.comparison}
\end{table*}

\subsection{Convergence Analysis for \texttt{FedNDES}}
To derive global convergence free from unknown problem parameters, we require a new condition.

  \begin{assumption}[Self-concordant function]
      \label{asm.self-concordant}
      A closed convex function $\varphi: \RR^d \to \RR$ is \textit{self-concordant} if $|\varphi'''(\ww)| \leq 2 (\varphi''(\ww))^{3/2}$.
      We assume the loss function $\ell$ is a convex self-concordant function.
  \end{assumption}
  
  The condition extends to the loss function $\ell: \RR^d \to \RR$ by imposing the requirement on the univariate function $\varphi_{\ww, \oo}(t) := \ell(\ww + t \oo)$ for $\ww, \oo$ in the domain of $\ell$.
  Examples for self-concordant functions include linear, quadratic functions, negative logarithm, and more examples can be found in \cite{lacotte2021adaptive}.
  
  \begin{definition}[Empirical effective dimension]
      If the regularity term is $\lambda$-strongly convex, the empirical effective Hessian dimension is defined as
      $
          \tilde{d}_\lambda := \sup_{\ww \in \mh} \text{\rm Tr}\left(\nabla^2 L(\md, \ww)(\nabla^2 L(\md, \ww) + \lambda I)^{-1}\right).
      $
  \end{definition}
  The empirical effective dimension $\tilde{d}_\lambda$ is substantially smaller than the feature space $M$ \cite{bach2013sharp,alaoui2015fast}.
  The effective dimension is related to the covariance matrix $\nabla^2 L(\md, \ww) = \frac{1}{N} \phi(\XX)^\top \phi(\XX)$, which has been well-studied for leverage scores sampling \cite{rudi2018fast,luise2019leveraging,chen2021fast} in low-rank approximation.
  Meanwhile, the expected effective dimension is defined as $\mathcal{N}(\lambda) = \text{Tr}(T_K (T_K + \lambda I)^{-1})$ based on the covariance operator $T_K = \int_X \langle \cdot, \phi(\xx) \rangle \phi(\xx) d \rho_X(\xx)$, which has been widely used to prove the optimal learning guarantees for the squared loss \cite{caponnetto2007optimal,smale2007learning}.
  
  Without initialization condition, we provide the following convergence guarantee for \texttt{FedNDES}.
  \begin{theorem}[Convergence guarantees of FedNDES]
      \label{thm.FedNDES}
      Let $\delta \in (0, 1)$ and the sketch matrices be SRHT.
      Under Assumptions \ref{asm.differentiable}, \ref{asm.self-concordant}, FedNDES updates in Algorithm \ref{alg.FedNDES}, then with a high probability, the number of iterations $T$ and the sketch size $k$ satisfying
      \begin{align*}
          T = \mO\left(\log \log \left(\frac{1}{\delta}\right)\right), \quad
          k = \Omega\left(\tilde{d}_\lambda\right),
      \end{align*}
      can obtain a $\delta$-accurate solution $\|\ww_t - \wkrr\| \leq \delta$ without any initialization condition.
      Here, $\wkrr$ is the centralized model on $\md$ and $\ww_t$ is the federated model trained in Algorithm \ref{alg.FedNDES}.
  \end{theorem}

  Compared to Theorem \ref{thm.FedNS}, the above theorem removes the initialization condition.
  More importantly, it reduces the sketch size from $M$ to $\tilde{d}_\lambda$, which is much smaller than $M$ and thus it is more practical in federated learning settings.
  Since the communicated local sketch square root Hessian $\HessiannoniidjSketch \in \RR^{k \times M}$, the communication complexity in each iteration is $\mO(\tilde{d}_\lambda M)$.
  For example, supposing $\tilde{d}_\lambda \gtrsim \log \log (1/\delta)$ as used in \cite{lacotte2021adaptive} with a standard learning rate $\delta = \mO(1/\sqrt{N})$, we obtain the sketch size $\tilde{d}_\lambda \gtrsim \log (0.5 \log N)$, which is significantly smaller than $M$.

  Since the generalization error $\delta$ and effective dimension $\tilde{d}_\lambda$ are relevant to the specific tasks, they are not estimated in Theorem \ref{thm.FedNDES}.
  However, they are important to measure the communication complexity in federated learning.
  It allows to provide accurate estimates for the error $\delta$ and the empirical effective dimension $\tilde{d}_\lambda$.

\begin{table}[t]
    \centering
    \small
    \begin{tabular}{lll|llll}
        \toprule
        Dataset &$n$ &$M$ &$k$ &$m$ &$\rho$ &$\alpha$ \\ \hline
        phishig &$11,055$ &$68$ &$17$ &$40$ &$0.1$ &$0.25$\\
        cod-rna &$59,535$ &$8$ &$10$ &$60$ &$30$ &$1$\\
        covtype &$581,012$ &$54$ &$20$ &$200$ &$50$ &$1$\\
        SUSY &$5,000,000$ &$18$ &$10$ &$1000$ &$50$ &$1$\\
        \bottomrule
    \end{tabular}
    \caption{Summary of datasets and hyperparameters.}
    \label{tab.exp}
\end{table}

\section{Compared with Related Work}

Table \ref{tab.comparison} reports the computational properties of related work to achieve $\delta$-accurate solutions.
In terms of the regularized least squared loss, we compare the proposed \texttt{FedNS} and \text{FedNDES} with first-order algorithms, and Newton-type FL methods.
\texttt{FedNS} applies to commonly used sketch approaches, e.g. sub-Gaussian, SRHT, and SJLT, while \texttt{FedNDES} only applies to SRHT.
Different sketch types leads to various sketch times on the $j$-th local machine, i.e., $\mO(n_j M k)$ for sub-Gaussian, $\mO(n_j M \log k)$ for SRHT and $\mO(\text{nnz}(\PhinoniidDj))$ for SJLT.

\textbf{Compared with first-order algorithms.} Federated Newton's methods converge much faster, $\mO(\log 1/\delta)$ v.s. $\mO(1/\delta)$. 
But the communication complexities of federated Newton's methods are much higher, at least $\mO(kMt)$, while it is $\mO(Mt)$ for FedAvg and FedProx.
The proposed \texttt{FedNDES} achieves balance between fast convergence rate and small communication complexity, of which the convergence rate is $\mO(\log 1/\delta)$ and the communication complexity is $\mO(\tilde{d}_\lambda M t)$.

\textbf{Compared with Newton-type FL methods.}
DistributedNewton \cite{ghosh2020distributed} and LocalNewton \cite{gupta2021localnewton} perform Newton's method instead of SGD on local machines.
However, they only utilized local information and implicitly assume the local datasets cross devices are homogeneous, which limits their application in FL.
In contrast, our proposed algorithms communicate local sketching Hessian matrices to approximate the global one, which are naturally applicable to heterogeneous settings.
More recently, there are three Newton-type FL methods:
\begin{itemize}
    \item \textbf{FedNL} \cite{safaryan2021fednl}  compressed the difference between the local Hessian and the global Hessian from the previous step, and transferred the compressed difference to the global server for merging. On the theoretical front, FedNL achieved at least the linear convergence $\mathcal{O}(\log 1/\delta)$ with the communication cost $O(M)$ per round.
    \item \textbf{FedNew} \cite{elgabli2022fednew} used ADMM to solve an unconstrained convex optimization problem for obtaining the local update directions ${\boldsymbol H}^{-1}_{\mathcal{D}_j, t} {\boldsymbol g_{\mathcal{D}_j, t}}$ and performed Newton's method by averaging these directions in the server. However, this work only proved the algorithm can converge but without the convergence rates.
    \item \textbf{SHED} \cite{fabbro2022newton} first performed eigendecomposition on the local Hessian and incrementally send the eigenvalues and eigenfunctions to the server. The local Hessians were recovered on the server to perform Newton's method. 
    The algorithm achieved sup-linear convergence with the total communications costs $\mathcal{O}(M^2)$.
\end{itemize}

These recent Newton-type FL methods usually admit linear convergence rates, while the proposed algorithms in this work reach super-linear convergences.

\begin{figure*}[t]
    \centering
    \subfigure{\includegraphics[width=0.245\linewidth]{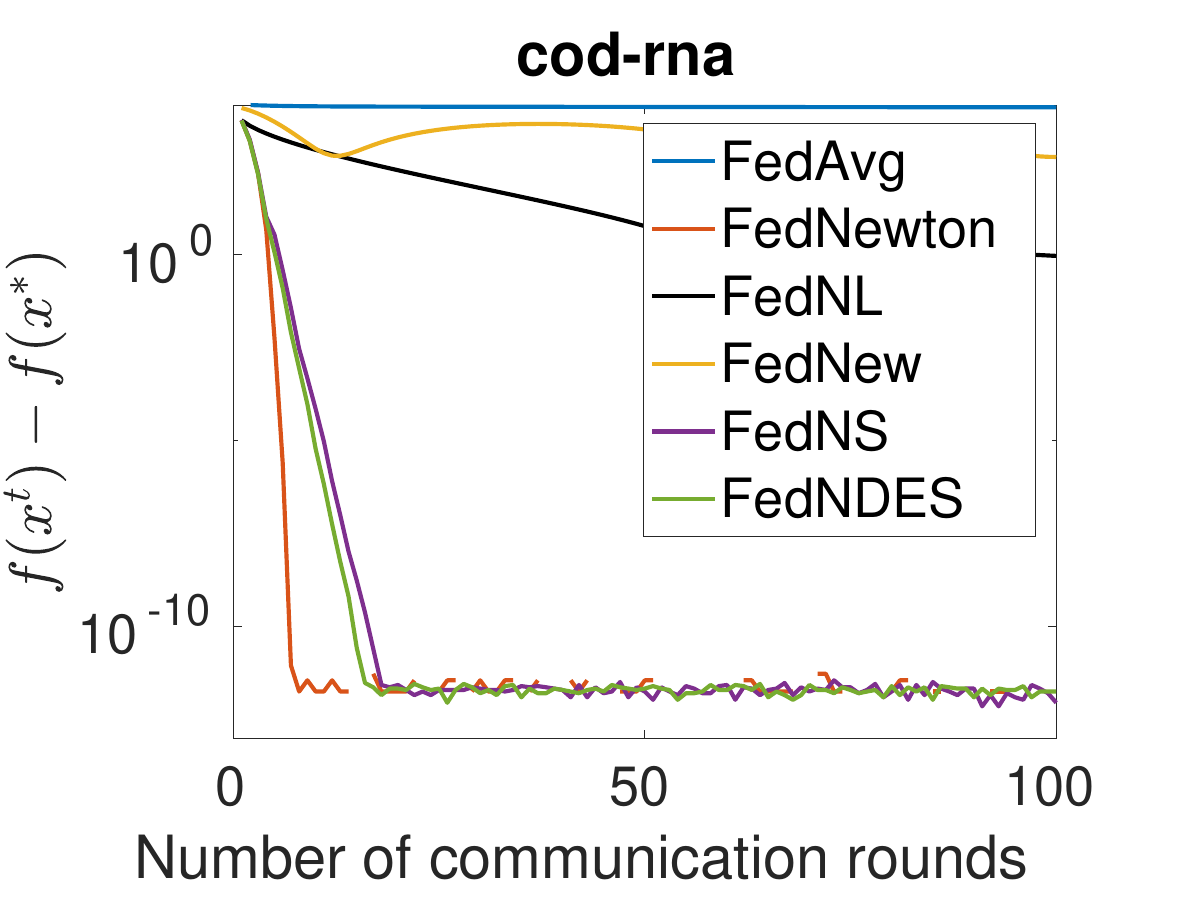}}
    \subfigure{\includegraphics[width=0.245\linewidth]{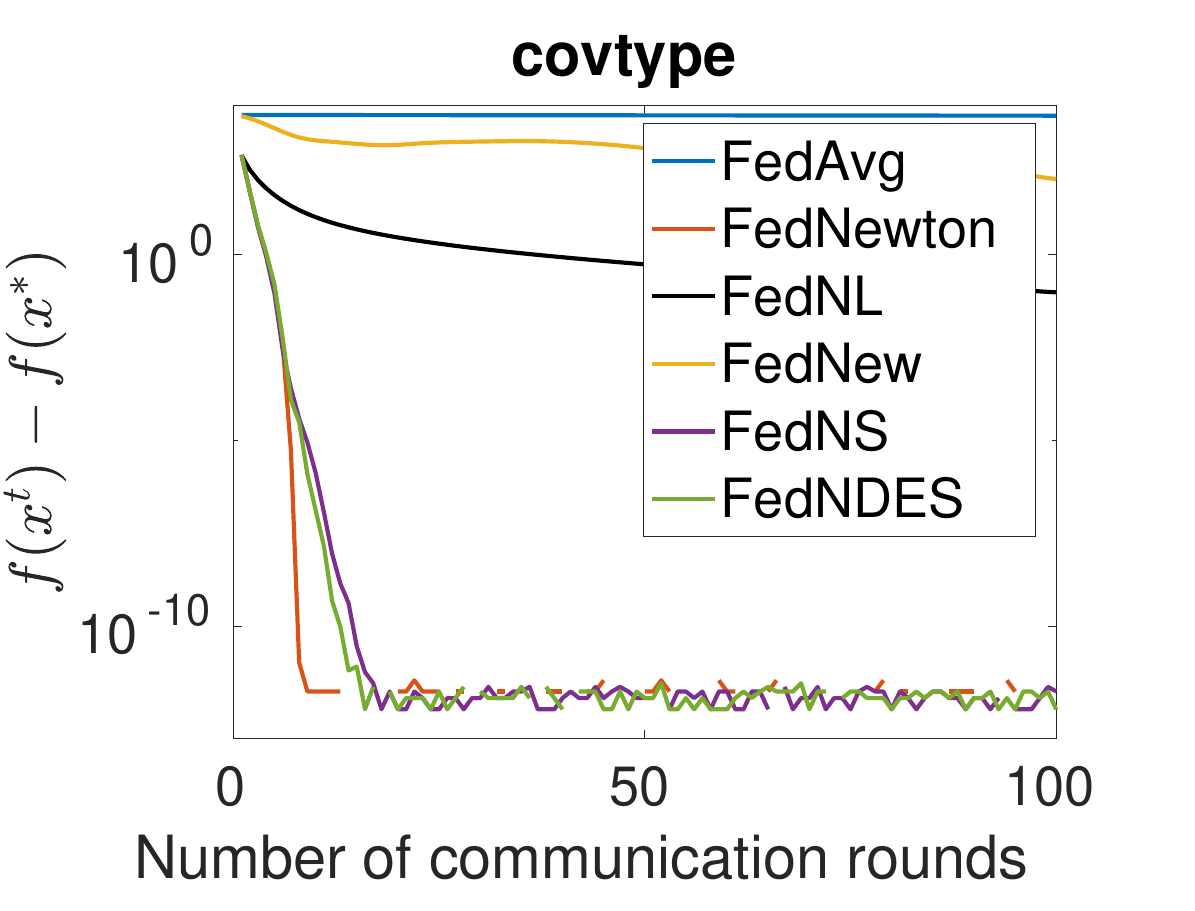}}
    \subfigure{\includegraphics[width=0.245\linewidth]{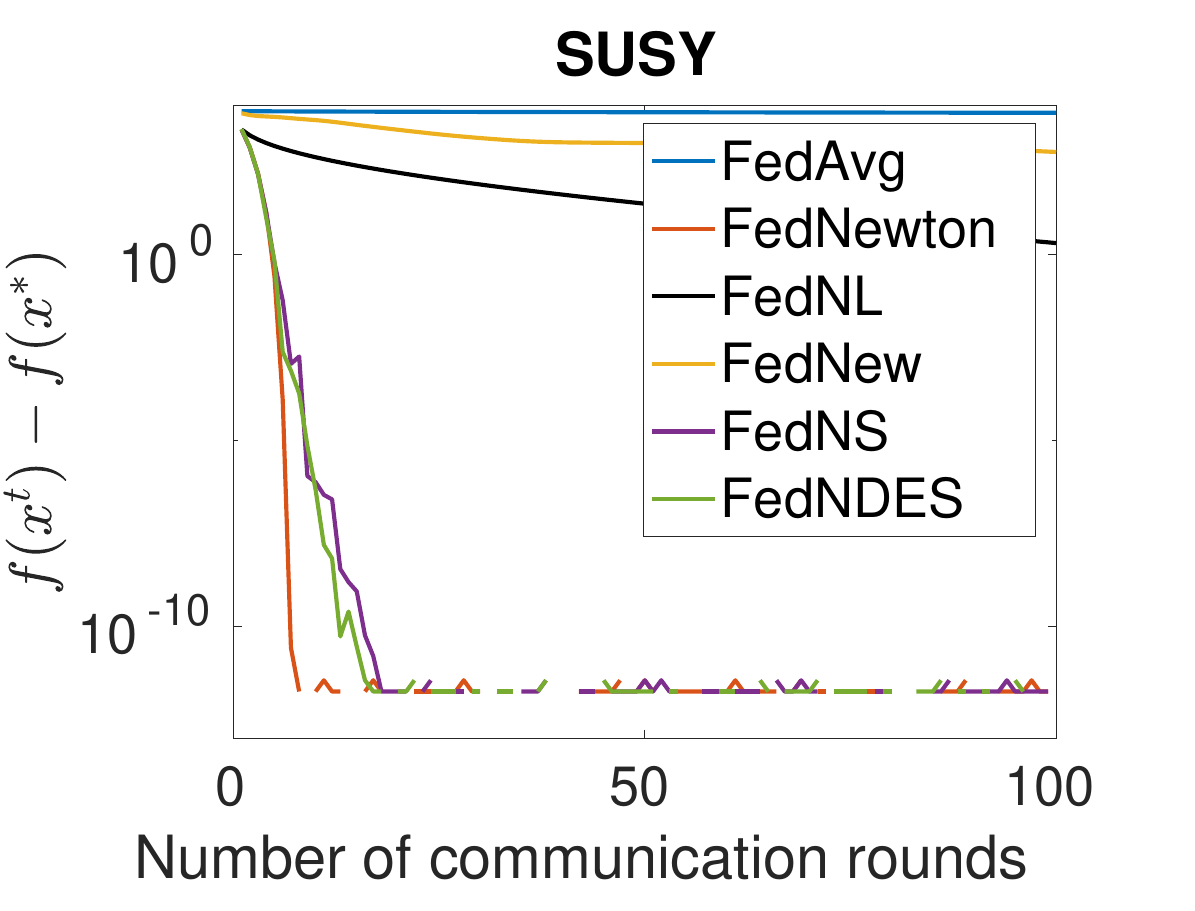}}
    \subfigure{\includegraphics[width=0.245\linewidth]{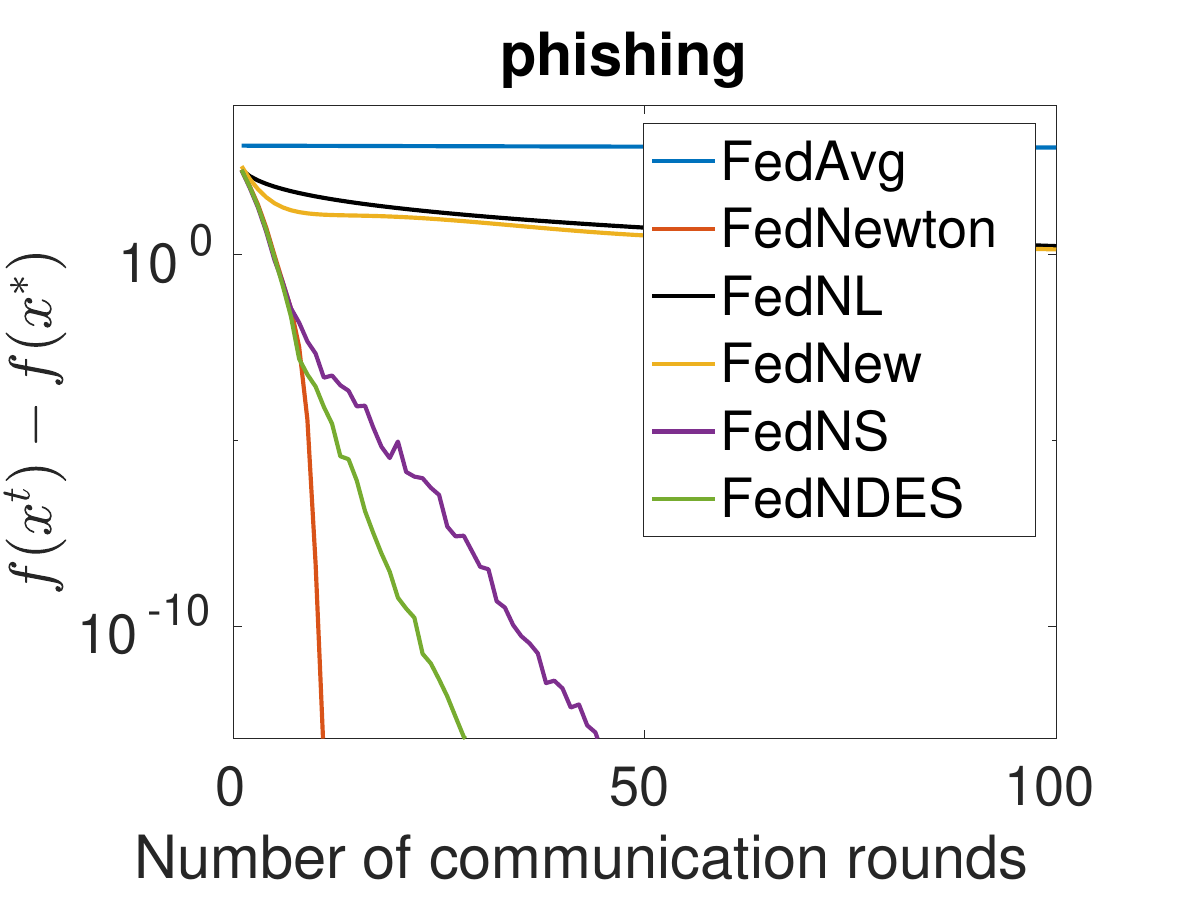}}
    \caption{The loss discrepancy between the compared methods and the optimal learner in terms of the number of communication rounds $t$.}
    \label{fig.exp.comparison}
\end{figure*}

\begin{figure*}[t]
    \centering
    \subfigure{\includegraphics[width=0.24\linewidth]{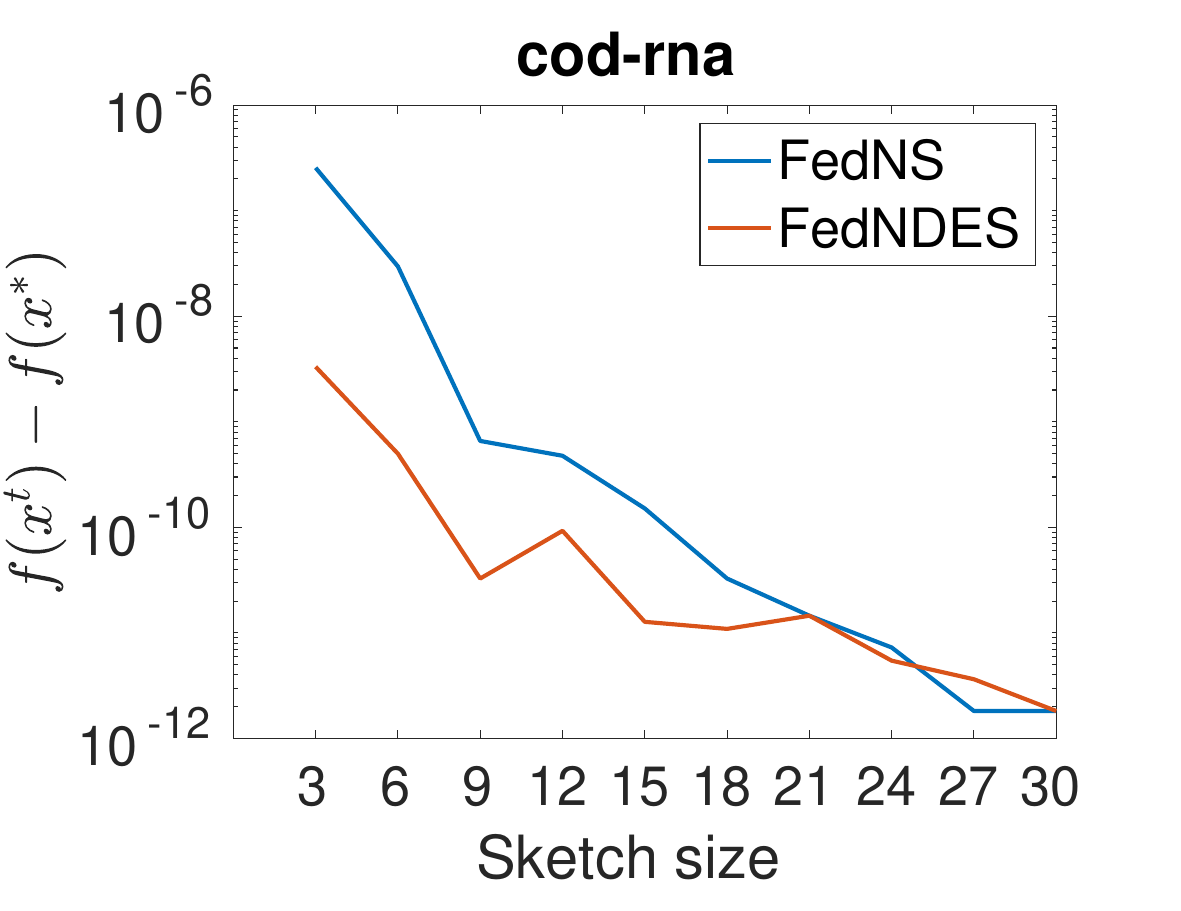}}
    \subfigure{\includegraphics[width=0.24\linewidth]{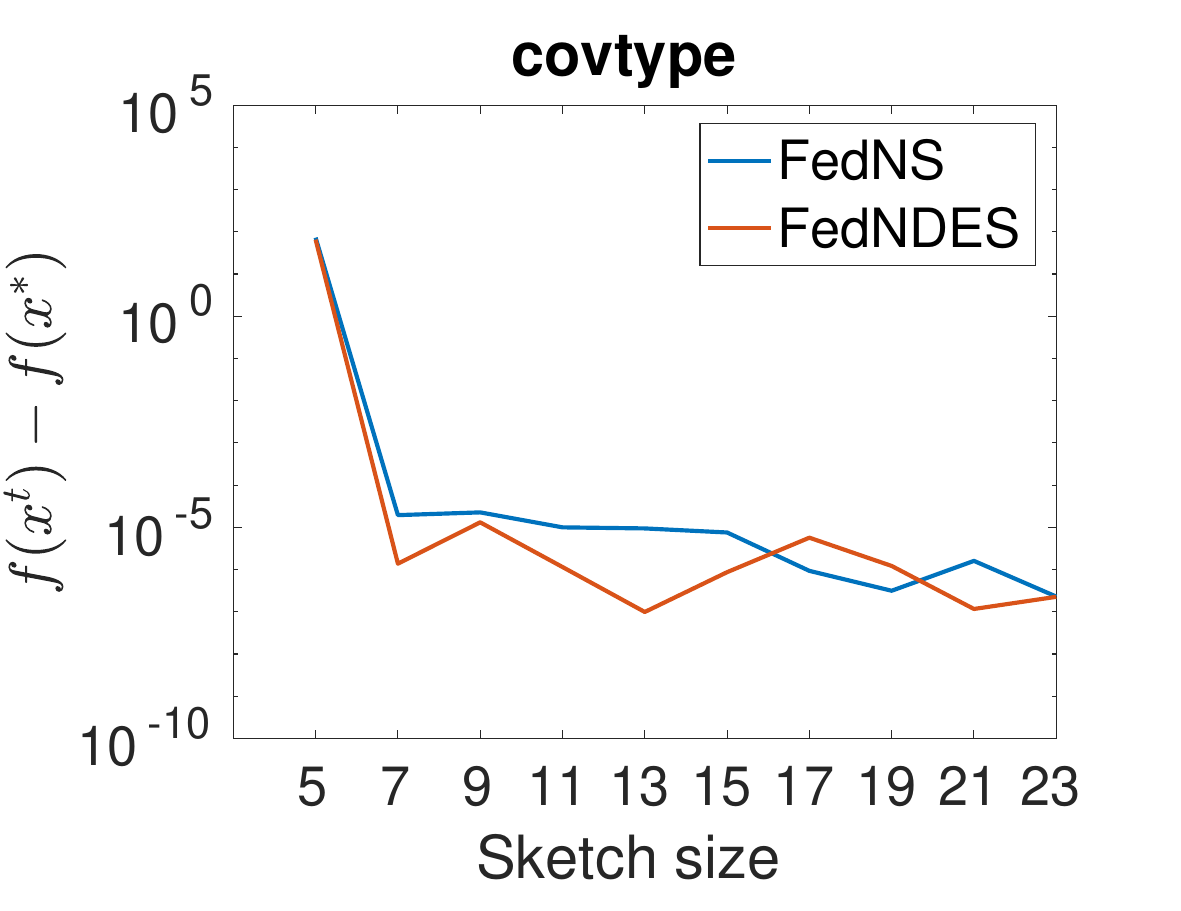}}
    \subfigure{\includegraphics[width=0.24\linewidth]{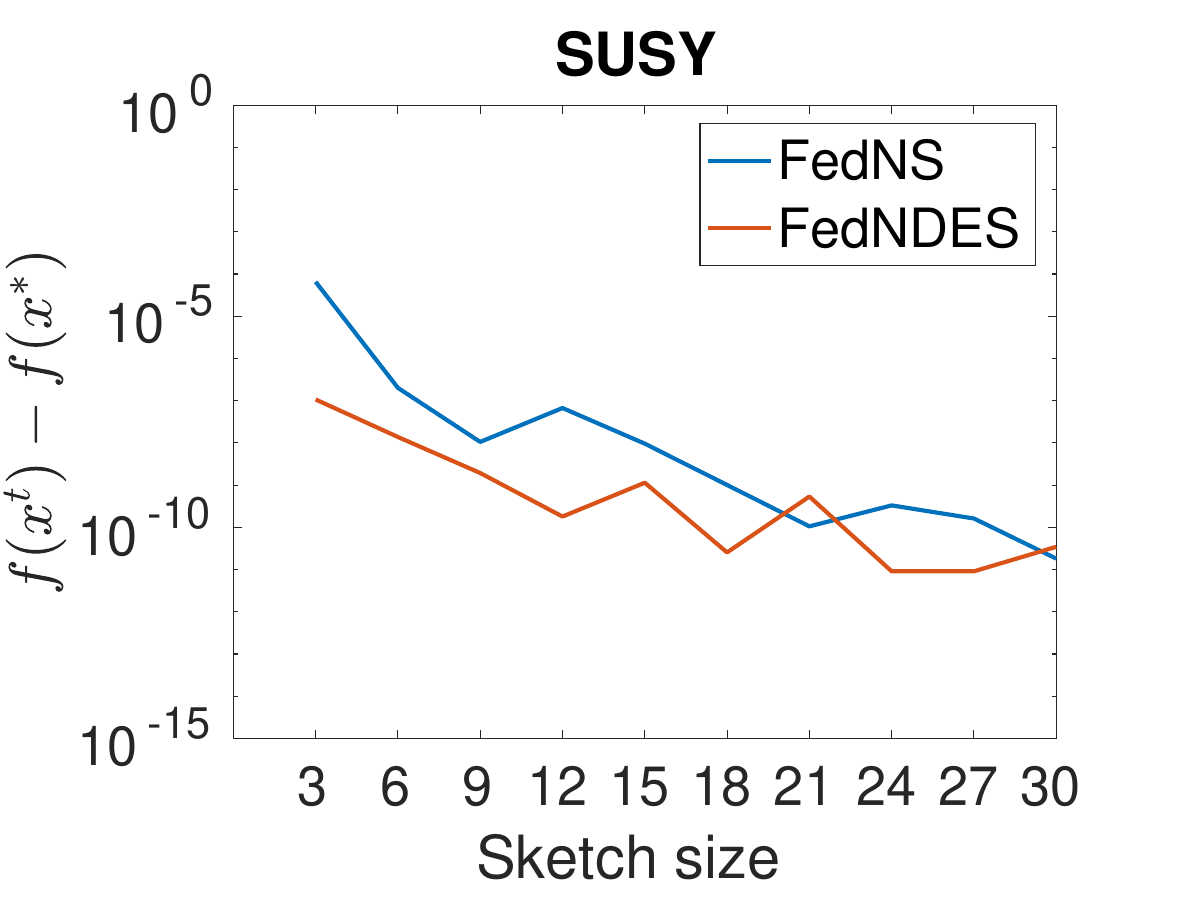}}
    \subfigure{\includegraphics[width=0.24\linewidth]{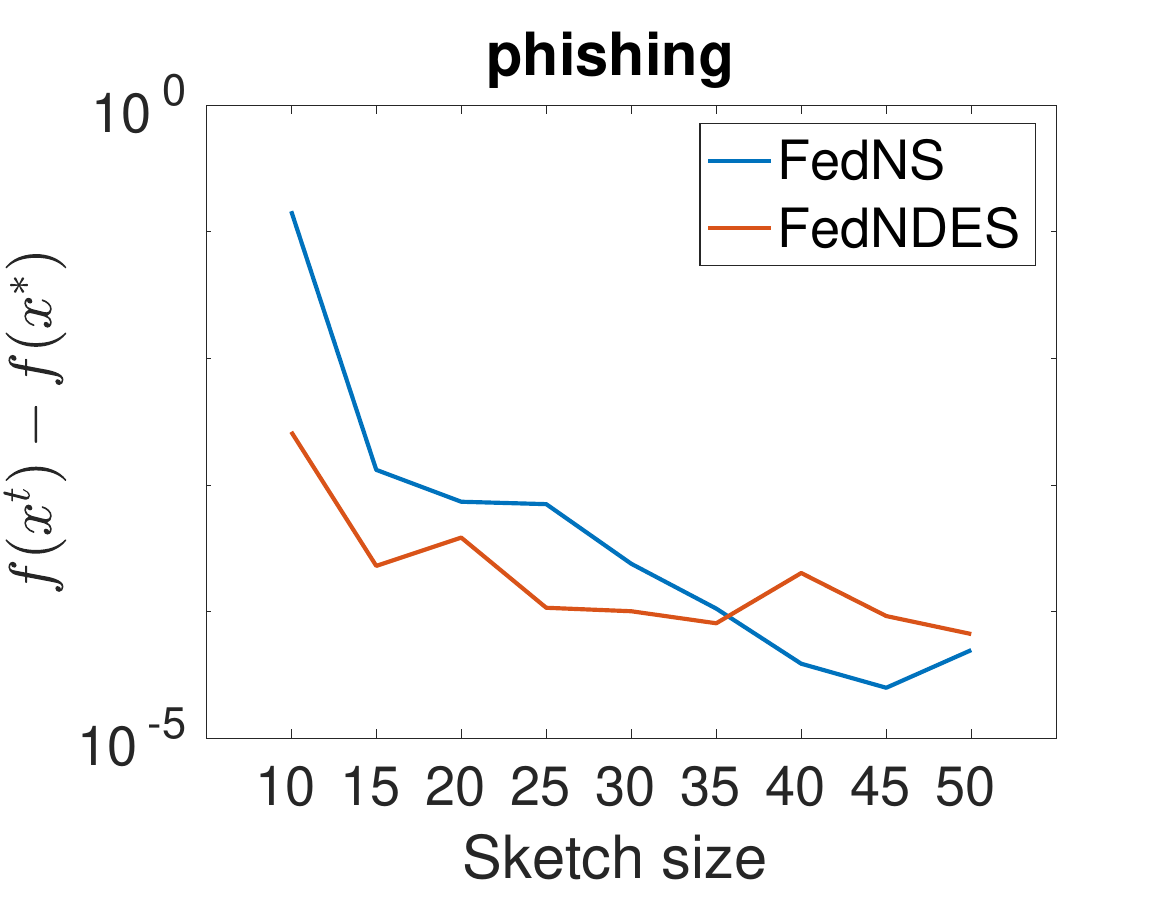}}
    \caption{The loss discrepancy between the compared methods and the optimal learner in terms of the sketch size $k$ on the datasets cod-rna, covtype, SUSY, and phishing}
    \label{fig.exp.sketch_size}
\end{figure*}

\section{Experiments}
In this section, we carry out experiments to corroborate our theoretical statements on several real-world federated datasets.
We implemented our methods by utilizing the public code from \cite{elgabli2022fednew} which includes FedNew \cite{elgabli2022fednew} and FedNL \cite{safaryan2021fednl}, while SHED algorithm was excluded due to the lack of public code \cite{fabbro2022newton}. The base model is a logistic regression and the algorithms update the Hessian at each iteration. 
We first explore the impact of sketching size on the proposed \texttt{FedNS} and \texttt{FedNDES}, and then compare related algorithms w.r.t. the communication rounds.

Following FedNew \cite{elgabli2022fednew}, we consider the regularized logistic regression $L(\md, \ww) = \frac{1}{N} \sum_{i=1}^{N} \log \left(1 + \exp(y_i \xx_i^\top \ww) \right) + \lambda \|\ww\|_2^2$, where $\lambda$ is a regularization parameter chosen to set $10^{-3}$.
All experiments are recorded by averaging results after $10$ trials and figures report the mean value.
We use the optimal gap $L(\ww^t) - L(\ww^*)$ as the performance indicator, where we use the global Newton's method as the optimal estimator $\ww^*$.
We evaluate the compared algorithms on public LIBSVM Data \cite{chang2011libsvm}.
We report the statistics of datasets and the corresponding hyperparameters in Table \ref{tab.exp}, where $\rho$ and $\alpha$ hyperparameters are used in FedNew.

\textbf{Convergence comparison.}
Figures \ref{fig.exp.comparison} reports the convergence of compared methods, demonstrating that: 
1) There is significant gaps between the convergence speeds of our proposed methods \texttt{FedNS}, \texttt{FedNDES} and the existing Newton-type FL methods, i.e. FedNew and FedNL. This validate the super-linear convergence of \texttt{FedNS} and \texttt{FedNDES}.
2) The proposed \texttt{FedNS} and \texttt{FedNDES} converge nearly to FedNewton, while FedNew and FedNL converge slowly closed to FedAvg.
3) \texttt{FedNDES} converges faster than \texttt{FedNS} and the final predictive accuracies of \texttt{FedNDES} are higher.
4) Even with small sketch sizes, the proposed sketched Newton-type FL methods can still preserve considerable accuracy.
5) Although our communication cost is higher $\mathcal{O}(Mk)$ than FedNew and FedNL, the number of communications is much smaller, resulting in lower total communications for our methods.

\textbf{Impact of sketch size on Performance}
Figure \ref{fig.exp.sketch_size} reports the predictive accuracies versus the sketch size, illustrating that 1) A larger sketch size always leads to better generalization performance. 
2) The proposed \texttt{FedNS} and \texttt{FedNDES} finally converges around the global Newton' method.
3) A small sketch size, i.e., $k \ll M$ and $k \ll N$, can still achieve good performance.
4) \texttt{FedNDES} obtains better generalization performance than \texttt{FedNS} with smaller sketch size.

\section{Conclusion}
Both convergence rate and communication costs are important to federated learning algorithms.
In this paper, by sketching the square-root Hessian, we devise federated Newton sketch methods, which communicate sketched matrices instead of the exact Hessian.
Theoretical guarantees show that the proposed algorithms achieve super-linear convergence rates with moderate communication costs.
Specifically, the sketch size of \texttt{FedNDES} can be small as the effective dimension of Hessian matrix.

Our techniques pave the way for designing Newton-type distributed algorithms with fast convergence rates. There are some future directions, including 1) One can employ adaptive effective dimension to effectively estimate the effective dimension of Hessian in practice \cite{lacotte2021adaptive}.
2) We consider the sparsification for the sketch Newton update in future \cite{derezinski2021newton} to further reduce the communication complexity.
3) The proposed methods and theoretical guarantees can be extened to decentralized learning \cite{hsieh2020non}.

\section{Acknowledgments}
The work of Jian Li is supported partially by National Natural Science Foundation of China (No. 62106257), and Project funded by China Postdoctoral Science Foundation (No. 2023T160680). The work of Yong Liu is supported partially by National Natural Science Foundation of China (No.62076234), Beijing Outstanding Young Scientist Program (No.BJJWZYJH012019100020098), the Unicom Innovation Ecological Cooperation Plan, and the CCF-Huawei Populus Grove Fund.

\bibliography{all}

\newpage
\appendix
\onecolumn

\section{Proofs}

\subsection{Convergence Analysis for \texttt{FedNS}}
\begin{proof}[Proof of Theorem \ref{thm.FedNS}]
    The main difference is that \texttt{FedNS} sketches local Hessian $\SS_j^t \Hessiannoniidjt^{1/2}$ on the loss function while \texttt{FedNS} directly sketches the global Hessian $\SS^t \Hessian^{1/2}$ on the objective.
    
    In Algorithm \ref{alg.FedNS}, we generalize local sketch matrices $(\SS_j^t)_{j=1}^m$ of the dimension $k \times n_j$ in an independent serialization.
    By concatenating local sketch matrices in the column direction and local square-root Hessian matrices in the row direction, we obtain
    \begin{align*}
        \SS^t &= [\SS^t_1, \cdots, \SS^t_m], \\
        \Hessian^{1/2} &= \left[\HH_{\md_1, t}, \cdots, \HH_{\md_m, t}\right]^\top.
    \end{align*}

    The above equations lead to 
    \begin{align*}
        \SS^t \Hessian^{1/2} = \sum_{j=1} \SS^t_j \Hessiannoniidjt^{1/2}.
    \end{align*}
    Therefore, the update of \texttt{FedNS} recovers the centralized Newton's method.

    From Corollary 1 \cite{pilanci2017newton}, the sketch size is lower bounded by the form of the squared Gaussian width, which is at most $\min\{N, M\}$.
    Since $N > M$ in federated learning, we have $k \gtrsim M$.
    The distance $\|\ww_t - \wkrr\|$ becomes substantially less than $1$ as the iteration increase.
    And then from Corollary 1 in \cite{pilanci2017newton}, considering the Newton sketch iterates using the iteration-dependent sketching accuracy $\epsilon = \frac{1}{\log(1+t)}$, it holds with the probability at least $1 - c_1 e^{-c_2 k \epsilon^2}$ that
    \begin{align*}
        &\|\ww_t - \wkrr\|_2 \\
        \leq &\frac{1}{\log(1 + t)} \frac{\beta}{\gamma} \|\ww_{t-1} - \wkrr\|_2 + \frac{4L}{\gamma} \|\ww_t - \wkrr\|_2^2.
    \end{align*}
    Note that from Lemma 1 in \cite{pilanci2017newton}, the sketch size satisfies $m \gtrsim \epsilon^{-2} M = \frac{1}{\log^2(1+t)} M$.
\end{proof}

\subsection{Convergence Analysis for \texttt{FedNDES}}
\begin{proof}[Proof of Theorem \ref{thm.FedNDES}]
    From Theorem 2 \cite{pilanci2017newton} and Lemma \cite{lacotte2021adaptive}, based on the backtracking parameters $(a, b)$ in Algorithm \ref{alg.FedNDES}, we define the parameters
    \begin{align*}
        \nu := ab\frac{\eta^2}{1+\left(\frac{1+\epsilon}{1-\epsilon}\right) \eta}, \qquad
        \eta := \frac{1}{8} \frac{1-\frac{1}{2}\left(\frac{1+\epsilon}{1-\epsilon}\right)^2 - a}{\left(\frac{1+\epsilon}{1-\epsilon}\right)^3}.
    \end{align*}
    Then, from Theorem 2 \cite{lacotte2021adaptive}, to obtain $\delta$-accurate solution with the probability at least $1 - p_0$, the number of total iterations $T$ should satisfy the condition 
    \begin{align*}
        T \leq \overline{T} := \frac{L(\ww_0) - L(\wkrr)}{\nu} + T_{\tau, \frac{3}{8}\delta} + 1,
    \end{align*}
    where $\lim_{\tau \to 0} T_{\tau, \frac{3}{8}\delta} \leq \frac{\log(8/3\delta)}{\log(25/16)}$.

    Meanwhile, two stages sketch sizes should satisfy 
    \begin{align*}
        &\overline{m}_1 \gtrsim \tilde{d}_\lambda + \log\left(\frac{\overline{T}}{p_0}\right)\log \left(\frac{\tilde{d}_\lambda \overline{T}}{p_0}\right), \\        
        &\overline{m}_2 \gtrsim \delta^{-\tau} \left[\tilde{d}_\lambda + \log\left(\frac{\overline{T}}{p_0 \delta^{\tau/2}}\right)\log \left(\frac{\tilde{d}_\lambda \overline{T}}{p_0}\right)\right].
    \end{align*}

    We consider the linear convergence case, i.e., $\tau=1$. 
    Using $\tau=0$ and ignoring the logarithm terms, we obtain $p_0 \asymp \frac{1}{\tilde{d}_\lambda}$, the sketch sizes $\overline{m}_1 \asymp \tilde{d}_\lambda$ and $\overline{m}_2 \asymp \frac{\tilde{d}_\lambda \log(\tilde{d}_\lambda/\delta)}{\delta}$, and the number of iterations $T \lesssim \log \log (\frac{1}{\delta})$.
\end{proof}

\subsection{Generalization Analysis for \texttt{FedNS} with the Squared Loss}
  \label{sec.krr}
  In the above sections, we present the converge analysis for \texttt{FedNS} and \texttt{FedNDES}, but the generalization analysis relies on the specific loss function.
  Here, we consider the squared loss with the RKHS and the ridge regularization, i.e. kernel ridge regression (KRR).
  Together with the classic integral operator theory, we derive the generalization error bound for KRR with the optimal learning rates.
  
  The target of regression learning is to find a predictor to approximate the true regression in the RKHS
  \begin{equation}
      \begin{aligned}
          \label{f.rho}
          &\frho(\xx)=\int_\mathcal{Y} y d\rho(y|\xx), \qquad \xx \in \mathcal{X}.
      \end{aligned}
  \end{equation}
  
  \begin{assumption}[Source condition]
      \label{asm.regularity}
      Define the integral operators $L: \Ltwo \to \Ltwo$,
      \begin{align*}
          (L g)(\cdot) = \int_\X \langle \phi(\cdot), \phi(\xx) \rangle g(\xx)d\rhox(\xx), \quad  \forall ~ g \in \Ltwo.
      \end{align*}
      Assume there exists $R>0$, $r \in [1/2, 1]$, such that
      \begin{align*}
          \|L^{-r}\frho\| \leq R.
      \end{align*}
      where the operator $L^r$ denotes the $r$-th power of $L$ as a compact and positive operator.
  \end{assumption}

\begin{assumption}[Capacity condition]
    \label{asm.capacity}    
    For $\lambda \in (0, 1)$, we define the effective dimensions as
    \begin{align*}
        \mathcal{N}(\lambda) = \text{Tr} (C (C + \lambda I)^{-1}), 
    \end{align*}
    Assume there exists $Q>0$ and $\gamma \in [0, 1]$, such that
    \begin{align*}
        \mathcal{N}(\lambda) \leq Q^2\lambda^{-\gamma}.
    \end{align*}
\end{assumption}

Both source condition and capacity condition are standard assumptions in the optimal statistical learning for the KRR related literature \cite{caponnetto2007optimal,smale2007learning,rudi2017generalization,lin2020optimal,liu2021effective}.
The effective dimension $\mathcal{N}(\lambda)$ measure the capacity of the RKHS $\mh$, and it is the expected version of $\tilde{d}_\lambda$ for KRR, depending on the distribution rather than the sample.

\begin{theorem}[Excess risk bound for FedNS with the squared loss]
    \label{thm.KRR}
    Under Assumptions \ref{asm.regularity}, \ref{asm.capacity}, $r \in [1/2, 1]$ and $\gamma \in [0, 1]$, then with a high probability, the number of iterations $T$ and the sketch size $k$ satisfying
    \begin{align*}
        T &= \mO\left(\frac{2r}{2r+\gamma}\log N\right), \\
        k &= \left\{
            \begin{array}{lcl}
            \Omega\left(M\right), & \qquad & \text{for \texttt{FedNS}}\\
            \Omega\left(N^\frac{\gamma}{2r+\gamma}\right), & & \text{for \texttt{FedNDES}}.
            \end{array}
        \right.
    \end{align*}
    can obtain the following generalization error bound
    \begin{align*}
        L(\ww_t) - L(\wrho) = \mO\left(N^\frac{-2r}{2r+\gamma}\right).
    \end{align*}
    Note that, $\wrho \in \mh$ is the RKHS model for the target predictor $\frho(\xx) = \langle \wrho, \phi(\xx)\rangle$.
\end{theorem}
\begin{proof}[Proof of Theorem \ref{thm.KRR}]
    From the error decomposition \eqref{eq.error_decomposition}, we have
    \begin{align}
        \label{eq.proof.error_decomposition}
        L(\ww_t) - L(\wrho)
        \leq L(\ww_t) - L(\wkrr) + L(\wkrr) - L(\wrho),
    \end{align}
    where $\wkrr = \left[\PhiD^\top \PhiD + \lambda N I\right]^{-1} \PhiD^\top \yy$ is the closed-form solution on the entire training data $\md$.

    Under Assumptions \ref{asm.regularity}, \ref{asm.capacity}, and setting $\lambda = N^\frac{-1}{2r+\gamma}$, the excess risk bound for centralized KRR is standard \cite{caponnetto2007optimal,smale2007learning}
    \begin{align}
        \label{eq.proof.centralized_excess_risk}
        L(\wkrr) - L(\wrho) = \mO \left(N^\frac{-2r}{2r+\gamma}\right).
    \end{align}

    We let $\delta \asymp N^\frac{-2r}{2r+\gamma}$, and then 
    the federated error holds with a high probability
    \begin{align}
        \label{eq.proof.federated_error}
        L(\ww_t) - L(\wkrr) = \mO \left(N^\frac{-2r}{2r+\gamma}\right)
    \end{align}
    
    From Theorem \ref{thm.FedNS}, since the square loss satisfy Assumptions \ref{asm.differentiable}, \ref{asm.convex}, \ref{asm.lipschitz}, the number iterations and the sketch size for \texttt{FedNS} achieve
    \begin{align}
        \label{eq.proof.federated_error.FedNS}
        T = \mO\left(\frac{2r}{2r+\gamma}\log N\right), \quad
        k = \Omega(M).
    \end{align}
    
    Similarly, from Theorem \ref{thm.FedNDES}, since the square loss satisfy Assumptions \ref{asm.differentiable}, \ref{asm.self-concordant}, the number iterations and the sketch size for \texttt{FedNDES} should satisfy 
    \begin{align}
        \label{eq.proof.federated_error.FedNDES}
        T = \mO\left(\frac{2r}{2r+\gamma}\log N\right), \quad
        k = \Omega\left(\tilde{d}_\lambda \right) = \Omega\left(N^\frac{\gamma}{2r+\gamma}\right).
    \end{align}
    The last step is due to $(1/3) N^\frac{\gamma}{2r+\gamma} \leq \tilde{d}_\lambda \leq  3 N^\frac{\gamma}{2r+\gamma}$ from Lemma 1 \cite{rudi2018fast} together with Assumption \ref{asm.capacity} and $\lambda = N^\frac{-1}{2r+\gamma}$.

    Substituting \eqref{eq.proof.centralized_excess_risk}, \eqref{eq.proof.federated_error}, \eqref{eq.proof.federated_error.FedNS} and \eqref{eq.proof.federated_error.FedNDES} to \eqref{eq.proof.error_decomposition}, we prove the theorem.
\end{proof}

The learning rate in the above generalization error bound is $\mO(N^\frac{-2r}{2r+\gamma})$, which is optimal in the minimax sense \cite{caponnetto2007optimal}.
Since both the number of iterations $T$ and the sketch size $k$ are estimated, we can compute the total time complexity and communication complexity.
For the sake of simplicity, we assume $n_1 = \cdots = n_m = N/m$.
For \texttt{FedNS}, the total computational complexity is $\mO(\max_{j\in[m]} {n_j} M d + n_j M \log M + mM^3 + M^2 \log N)$ and the communication complexity is $\mO(M^2 \log N)$.
For \texttt{FedNDES}, the total computational complexity is $\mO(\max_{j\in[m]} {n_j} M d + n_j M \log N + mN^\frac{\gamma}{2r+\gamma}M^2 + M^3 + M^2 \log N)$ and the communication complexity is $\mO(M N^\frac{\gamma}{2r+\gamma} \log N)$.
Note that, when $M \leq N^\frac{\gamma}{2r+\gamma}$, \texttt{FedNS} obtain smaller complexities but it requires the initialization condition.

In the worst case $(r = 1/2, \gamma = 1)$, without Assumptions \ref{asm.regularity}, \ref{asm.capacity}, the sketch size is $k = \Omega(\sqrt{N})$ to achieve the optimal learning rate. 
In the benign case $\gamma \to 0$, a constant number of sketch size $k = \Omega(1)$ is sufficient to guarantee the optimal rate.

\end{document}